\documentclass[twoside,11pt]{article}
\usepackage{authblk}

\usepackage{amsmath}
\usepackage{amsfonts}
\usepackage{amssymb}
\usepackage{amsthm}
\usepackage{graphbox}
\usepackage{adjustbox}
\usepackage{natbib}

\usepackage{verbatim}
\usepackage{subfigure}
\usepackage{mathtools}
\usepackage{url}

\usepackage{tikz}
\usetikzlibrary{positioning}

\usepackage[T1]{fontenc}

\newcommand{\bm}[1]{\mathbf{#1}}

\newcommand{\tsetsize}{n}
\newcommand{\esize}{e}

\newcommand{\labelseq}{\mathbf{y}}
\newcommand{\binseq}{\mathbf{b}}
\newcommand{\binvar}{B}
\newcommand{\labelvar}{Y}
\newcommand{\predfun}{f}

\newcommand{\lalg}{\mathcal{A}}

\newcommand{\posc}{w}
\newcommand{\uval}{U}
\newcommand{\ureal}{u}
\newcommand{\uu}{W}
\newcommand{\maxulight}{L}
\newcommand{\johnson}{J}

\newtheorem{fct}{Fact}
\newtheorem{theorem}{Theorem}
\newtheorem{lemma}{Lemma}
\newtheorem{proposition}{Proposition}
\newtheorem{corollary}{Corollary}
\newtheorem{definition}{Definition}
\newtheorem{remark}{Remark}
\newtheorem{example}{Example}

\newcommand{\cwwset}{S}
\newcommand{\indset}{\bm{x}}
\newcommand{\hdist}{H}
\newcommand{\cwc}{A}
\newcommand{\degree}{d}
\newcommand{\vecset}{C}

\newcommand{\ucount}{R}
\newcommand{\kernel}{k}
\newcommand{\critval}{c}

\begin{document}

\title{A Link between Coding Theory and Cross-Validation with Applications}

\author[1]{Tapio Pahikkala \thanks{ aatapa@utu.fi}}
	\author[1]{Parisa Movahedi}
	\author[1]{Ileana Montoya Perez} 
	\author[2]{Havu Miikonen}
	\author[1,3]{Stephan Foldes}
\author[1]{Antti Airola}
	\author[1]{Laszlo Major}

\affil[1]{Department of Computing,
University of Turku, Finland}
\affil[2]{Department of Mathematics and Statistics, University of Turku, Finland}
\affil[3]{University of Miskolc, Hungary}
\maketitle

\begin{abstract}
How many different binary classification problems a single learning algorithm can solve on a fixed data with exactly zero or at most a given number of cross-validation errors? While the number in the former case is known to be limited by the no-free-lunch theorem, we show that the exact answers are given by the theory of error detecting codes. As a case study, we focus on the AUC performance measure and leave-pair-out cross-validation (LPOCV), in which every possible pair of data with different class labels is held out at a time. We show that the maximal number of classification problems with fixed class proportion, for which a learning algorithm can achieve zero LPOCV error, equals the maximal number of code words in a constant weight code (CWC), with certain technical properties. We then generalize CWCs by introducing light CWCs, and prove an analogous result for nonzero LPOCV errors and light CWCs. Moreover, we prove both upper and lower bounds on the maximal numbers of code words in light CWCs. Finally, as an immediate practical application, we develop new LPOCV based randomization tests for learning algorithms that generalize the classical Wilcoxon-Mann-Whitney U test.
\end{abstract}


\section{Introduction}

Cross-validation is a commonly used method for estimating prediction performance in supervised learning when the amount of available data is limited. With cross-validation, a natural question is whether it would be possible to obtain equally good or better results if the labels of the data were randomly assigned. In this work, we seek a general answer to this question such that would hold over all possible learning algorithms for the specific case of area under ROC curve (AUC) estimation via \textbf{leave-pair-out cross-validation} (LPOCV) (see e.g. \citet{airola2011experimental}). Similarly to the infamous \textbf{no-free-lunch (NFL) theorems} (originally formulated by \citet{Wolpert1992connection,Wolpert1996nfl}, see also \citet{wolpert2021important,sterkenburg2021no} for recent comprehensive reviews), we show that the number of classification problems for which a machine learning method has exactly zero LPOCV errors can be bounded by the maximal sizes of error detecting codes of the same code word length as the amount of data in LPOCV. Furthermore, we establish bounds on the maximum number of classification problems for which a machine learning method incurs only a specified number of LPOCV errors. These results allow designing new kinds of tests of statistical significance for AUC-estimation via cross-validation that hold for all possible learning algorithms. More generally, they provide novel ways to analyze learning algorithms in terms of their cross-validation error distributions.

We consider the following NFL framework that is a specialization of the one considered by \citet{Wolpert1996nfl,wolpert2021important}. We also refer to \citet{ho2002simple,lattimore2013no} for similar setups. We start from an arbitrary sequence $\indset=(x_1,\ldots,x_\tsetsize)$ of inputs, that can be basically anything, ranging from images or text documents to persons.\footnote{The indices of the sequence can be interpreted, for example, as the observation order of the data, and $\tsetsize$ as the overall number of observations, including the ones done in the future, but the indexing can in principle be any arbitrary fixed order of data. We only consider here a fixed sequence of finite length, because in the real world we only observe finite amounts of data (see e.g. \citet{Wolpert2002supervisednfl} for a similar reasoning).} The sequence can in principle have multiple occurrences of the same input but this is here considered irrelevant, since we always refer to the sequence entries by their indices that are distinct by their definition. This sequence is associated with a binary vector valued random variable $\labelvar$ of length $\tsetsize$.
Using the terminology by \citet{wolpert2021important}, its realization $\labelseq$, the sequence of binary class labels of the $\tsetsize$ data, is drawn from some unknown \textbf{distribution of learning problems} (DLP). In this context, the DLP can hence be considered simply as a probability mass function over all possible binary sequences of length $\tsetsize$.

We further assume that a subsequence of $\labelseq$ of nonzero length $\esize$ is hidden from the experimenter and the remaining $\tsetsize-\esize$ are revealed. The experimenter learns from the sequence of data and the revealed labels a hypothesis about the hidden labels. 
The learning is carried out by a \textbf{machine learning algorithm}, that can be considered as a mapping from the data and known labels, referred to as training data in the literature, to the sequence of $\esize$ predictions of the unknown labels.\footnote{In the usual setup in machine learning, also the inputs corresponding to the unknown labels are assumed to be unknown during the training phase and only revealed when the prediction is carried out. However, this technical detail is irrelevant for the forthcoming analysis and we pay no further attention to it.} Unless stated otherwise, we assume the learning algorithms to be both deterministic and symmetric. This indicates that the predictions are invariant to both re-running the learning algorithm and permuting the training data.

The simplest and most popularly used way for measuring how well the predictions match with the unknown labels is to count the fraction of incorrectly classified data with hidden labels, referred to as the \textbf{classification error}. Below, we shift our focus on the error based on the area under ROC curve, but we first consider the classification error in order to link our research with relevant literature. The \textbf{expected error} of the learning algorithm, with regard to the underlying DLP, is the error on the data with hidden labels averaged over all possible learning problems weighed by their probabilities in DLP. 

Next, we present a practical example that resembles the ones often used to illustrate the NFL theorem for machine learning.

\begin{example}
In a book recommendation mobile app
the sequence of data $\indset$ correspond to different books and the binary labeling of data indicates whether the user of the app likes a book or not. The user has input their preferences
for a subset of books, and the learning task is to predict these for the remaining ones. Due to privacy regulation no information about the preferences of any other user is available for the app, but rather the task needs to be solved for each user independently. The labeling corresponding to a single user is drawn from the underlying DLP, that can thus be interpreted as the distribution of different users' preferences over the books, so that some label sequences are more likely than others. 


Roughly put, a ML algorithm is good for this family of problems if it is able to learn well such user preferences that have a high probability of being drawn from the underlying DLP. Indeed, it is possible for a single learning algorithm to perfectly learn to predict the preferences for several different label sequences. From a training sequence of size $\tsetsize-\esize$, the maximum amount is $2^{\tsetsize-\esize}$ if this many users happen to differ in their preferences on books with revealed labels. Accordingly, one can interpret a learning algorithm as a catalog of $2^{\tsetsize-\esize}$ key-value pairs with keys corresponding to the subsequence of known labels and the values to the predictions of the remaining labels.
\end{example}

According to NFL, if DLP is equal to a uniform distribution, indicating that all $2^\tsetsize$ possible binary classification problems have an equal probability, then learning is impossible in terms of expected classification error. Indeed, it is straightforward to see that the expected classification error for any learning algorithm is exactly $0.5$, when one averages it uniformly over all $2^\tsetsize$ binary labelings. In particular, if the support of the underlying DLP would cover only a pair of learning problems having the same labels for the training data but exactly opposite labels for the remaining data, then the expected classification error of any learning algorithm on this pair of problems is 0.5. In contrast, if the underlying DLP has different probabilities for at least one pair of learning problems of the considered type, then there exists a learning algorithm with better than $0.5$ expected classification error. Accordingly, again using the terminology by \citet{wolpert2021important}, such a learning algorithm is \textbf{well-aligned} with the DLP it is facing. Meanwhile, the algorithm that would do exactly opposite predictions would be badly aligned with the DLP. Taking this consideration even further, if the support of DLP does not include any pair of problems such that would have the same labeling of training data but different for the remaining data, then there exists a learning algorithm perfectly aligned with the DLP, such that always outputs correct predictions with any labeling of the training data.

How the predictions are determined for the data with unknown labels based on the labeled data, is in the machine learning literature called inductive bias of the learning algorithm. \citet{mitchell1980need} defines inductive bias as ``any basis for choosing one generalization over another, other than strict consistency with the instances''. While our main focus in this paper is on the deterministic learning algorithms, we here mention for the sake of completeness the non-deterministic learning algorithm with the so-called maximum entropy prior. With this, we refer to the only learning algorithm with no inductive bias whatsoever, indicating that all predictions for the unknown labels are equally probable with all training data. With this algorithm,
learning is impossible, since no training set can have any effect on the inferred hypothesis that always treats both labels as equally probable for any data. Hence, this learning algorithm is the only one that is not well-aligned with any DLP. Thus, learning can be unsuccessful due to the properties of either DLP or learning algorithm. Beside these extremities, all learning algorithms are at least somewhat well aligned with some family of DLPs and at least somewhat badly aligned with the rest.

An important variation of this setup and also the main focus of this paper is the special case, in which the support of the underlying DLP for sequences of $\tsetsize$ data covers only label sequences that have exactly $\posc$ nonzero and $\tsetsize-\posc$ zero entries. Such binary sequences are in the literature referred to have a \textbf{constant weight}, since their Hamming weight, the sum of their entries, is a constant. We refer to a DLP for label sequences of length $\tsetsize$ with such a weight constraint as \textbf{distribution of constant weight learning problems (DCWLP)}. Moreover, we say that DCWLP of length $\tsetsize$ and weight $\posc$ is uniform if all learning problems covered by it (e.g. all binary label sequences of length $\tsetsize$ and weight $\posc$) are equally likely.

If $\posc\neq \tsetsize-\posc$, then a classifier that always predicts the majority class, obtains better than $0.5$ classification error when averaged over uniform DCWLP. This is of course perfectly in line with the NFL results, since this type of majority classifiers are well-aligned with this subset of learning problems. However, their classification errors tell nothing on how well they can distinguish the two classes from each other. Therefore, a more appropriate and popular way for measuring the discriminatory power in this case is what
is in the literature referred to as the \textbf{area under a receiver operating characteristic curve (AUC)} (see e.g.  \citet{hanley82auc}).
When AUC is considered, we assume that real-valued predictions are made for the data with unknown labels rather than binary classifications. AUC is often interpreted as the probability of the predicted real value for a randomly chosen observation labeled with 1 being larger than that for a randomly chosen observation labeled with 0. Accordingly, AUC can be conveniently expressed in terms of what we refer to a \textbf{pairwise error} on differently labeled data. Let us define
\begin{align}\label{pairwiseerror}
\kernel_\predfun(\binseq,i,j)=
\left\{
\begin{array}{cc}
	1&\predfun(i)<\predfun(j)\\
	0.5&\predfun(i)=\predfun(j)\\
	0&\predfun(i)>\predfun(j)
\end{array}
\right.\;,
\end{align}
where $\binseq$ is a binary sequence, $(i,j)$ is any pair of indices such that $b_i=1$ and $b_j=0$, and $\predfun$ is a  prediction function that maps the indices to real values.\footnote{The usual practice in machine learning is to define prediction functions as functions of the inputs $x$ rather than their indices but we decided to adapt the latter to reduce unnecessary complexity in notation.}
Then, AUC is the average value of $1-\kernel_\predfun$ over all index pairs conformable with $\binseq$. For example, if evaluated over the hidden subsequence of labels of length $\esize$, AUC is the average over all possible index pairs in the subsequence consisting of a datum labeled with 1 and another one with 0. Just like the expected classification error for uniform DLPs is 0.5 for all learning algorithms, the \textbf{expected AUC} for uniform DCWLP is 0.5 for all learning algorithms.

So far we only considered a fixed partition of the $\tsetsize$ data, such that the labels of only a certain subsequence of $\labelseq$ of length $\tsetsize-\esize$ are known and the remaining are to be predicted. That is, the expected classification error is conditional to this particular split of data. From now on, we drop this condition, and consider the expected classification error averaged over all possible splits of the data into $\tsetsize-\esize$ revealed and $\esize$ hidden.\footnote{The uniformity over the possible splits is natural, for example, if the contents of the input sequence $\indset$ are assumed to be independently drawn from some unknown distribution over their values. See e.g. \citet{wolpert2021important} for similar averaging.} This leads us to the following analogy between the considered DLPs and \textbf{error detecting codes}. A binary code that can detect up to $\esize$ erroneous bits from a binary sequence of length $\tsetsize$ is a set of \textbf{code words} (here binary sequences of length $\tsetsize$), such that the Hamming distance between every two distinct words is at least $\esize+1$. Moreover, a special case of binary codes, a code whose words are restricted to have exactly $\posc$ ones out of the $\tsetsize$ elements is called a  \textbf{constant-weight code} (CWC). The error detection procedure simply checks whether the binary sequence under consideration matches with any of the code words. If there is no match, an error is detected. Now consider a DLP from which one draws a label sequence of length $\tsetsize$ of which $\esize$ randomly chosen entries are hidden. If the support of DLP forms a binary error detecting code of length $\tsetsize$ and detecting capability of $\esize$ bits, there exists a learning algorithm with zero expected classification error for that DLP for all possible splits of the data into $\tsetsize-\esize$ known and $\esize$ unknown. Namely, the learning algorithm simply predicts the unknown labels according to the code word that the known labels match with. The converse is equally straightforward. If the set of label sequences with nonzero probability does not form an above kind of error detecting code, then there can not exist such a learning algorithm with zero expected classification error for that DLP for all train-test splits.

The above analogy between error detecting codes and learning algorithms with zero expected classification error manifests itself in a smaller scale on a limited data if a part of known labels are intentionally hidden for performance evaluation purposes. Indeed,
\textbf{cross-validation} (CV) is a family of learning algorithms' prediction performance evaluation techniques that simulate the above described setting. Subsequences of fixed length $\esize$ are held out for performance evaluation, while the remaining $\tsetsize-\esize$ data are used for training. Results of several (or even all possible) such train-test splits are averaged for obtaining a statistic of the prediction performance.\footnote{While CV is most often used to estimate the generalization performance outside the data used for CV, in this paper we rather focus on the above described question about whether the underlying DLP is uniform, and hence whether learning is possible in the first place. These are two very different statistical questions.}
In this paper, we focus especially on a certain type of CV useful for measuring AUC. We consider the \textbf{leave-pair-out cross-validation} (LPOCV), in which every differently labeled pair of data is hidden at a time (see e.g. \citet{airola2011experimental,smith2014correcting}).

Next, we link our study to statistical hypothesis testing, in a setup that resembles our consideration very closely and can be considered as its special case. We recall a classical tool known as the WMW test of statistical significance, Wilcoxon's rank sum test and the Mann-Whitney U-test. Given that the considered sequence of inputs $\indset$ is assumed fixed, we focus on the so-called randomization model interpretation of the test (see e.g. \citet{fay2010}). In the context of the test, we temporarily restrict our consideration on fixed sequences $\bm{f}=(\predfun(1),\ldots,\predfun(\tsetsize))\in\mathbb{R}^\tsetsize$ of predictions. In contrast to the above framework for learning algorithms, the prediction sequence under consideration is not learned from any part of the data, but it is rather given a priori. This can, in fact, be considered as a special case of a learning algorithm that ignores the known training labels and always outputs the predictions according to $\bm{f}$.

The so-called null hypothesis of the WMW test indicates that the predictions $\bm{f}$ are independent on the labels $\labelseq$. This independence is, in turn, equivalent with the assumption of uniformity of the underlying DCWLP. The strength of evidence against this null hypothesis is measured by what we refer to simply as the $\uval$-value. It is here defined as the sum over all possible evaluations of $\kernel_\predfun(\labelseq,i,j)$ on the data, where $i$ and $j$ go through the entries of $\labelseq$ labeled with 1 and 0, respectively. That is, $U$-value is the number of pairwise mistakes of $\bm{f}$ with respect to the labelling $\labelseq$. Then, AUC equals $1-\frac{U}{\posc(\tsetsize-\posc)}$ for data  of length $\tsetsize$ and weight $\posc$. Given a $U$-value and the values of $\tsetsize$ and $\posc$, one can determine the probability, referred to as p-value, of obtaining as small or even smaller $U$-value under the null hypothesis. The test rejects if the p-value is smaller than some pre-determined significance threshold $\alpha$. In other words, the threshold is an upper bound for the probability of falsely rejecting the null hypothesis when it holds. These probabilities with the WMW statistic are well-known and easy to calculate for small numbers of data. For example, if we have a sample of size 30 of which 15 are labeled with 1, and the labels of the data are drawn from uniform distribution over binary sequences of length 30 and weight 15, the probability mass function of the $\uval$-values is known to have the bell-shaped form depicted in Figure~\ref{figureofwilcoxondistribution}. Given a significance level, say $\alpha=0.05$, the number of data $\tsetsize$ and weight $\posc$, one can determine the maximal $U$-value for which the test still rejects. For the WMW test, these so-called \textbf{critical values} for different combinations of $\tsetsize$ and $\posc$ are often conveniently listed in statistical tables, as shown in Figure~\ref{wmwcriticals}.

\begin{figure}
	\centering
	\includegraphics[width=0.75\linewidth]{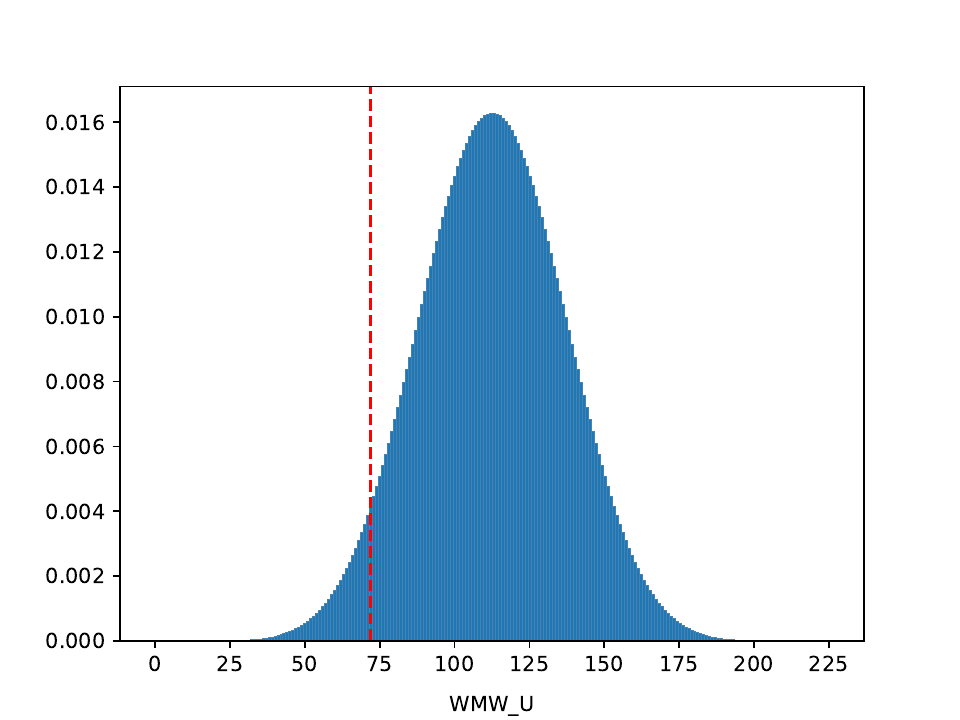}
	\caption{The distribution of U-values under the null hypothesis for a sample of size 30 of which 15 labeled with one and a fixed prediction function. This is also known as the Wilcoxon's distribution and it determines the critical values for the WMW-test. The red dashed line denotes the significance level 0.05, that is, 5\% of the probability mass is on the left side of the line.}
	\label{figureofwilcoxondistribution}
\end{figure}

\begin{figure}
	\centering
	\includegraphics[width=0.9\linewidth]{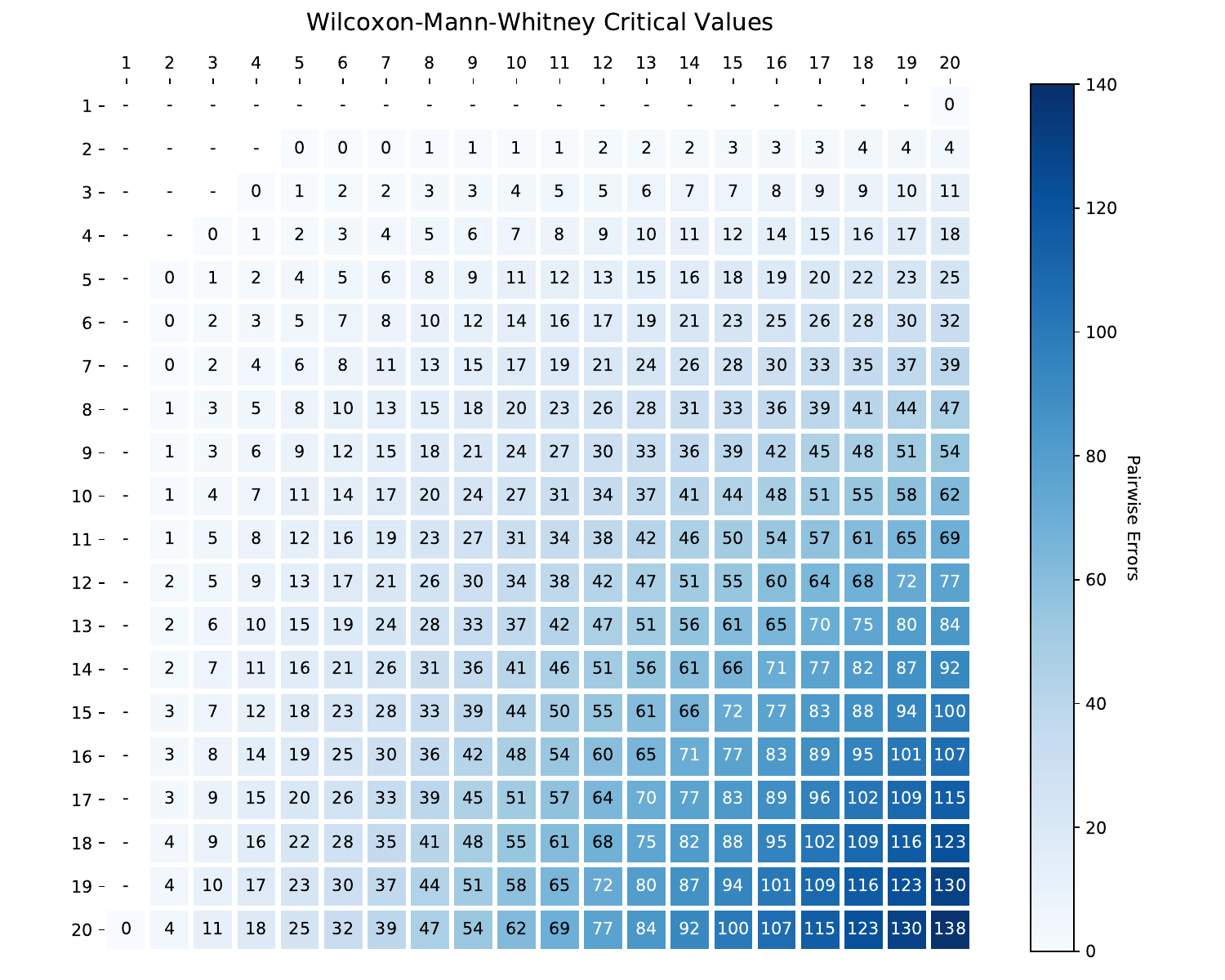}
	\caption{Critical values of the one-sided Wilcoxon-Mann-Whitney U test for a significance level 0.05, indicating for different numbers of data labeled with one and zero, what is the maximum number of pairwise errors with which the null hypothesis can still be rejected.}
	\label{wmwcriticals}
\end{figure}

Now we switch our focus back on learning algorithms and consider whether the number of LPOCV errors can be used as a tool for gathering evidence against the uniformity of DCWLP in the same sense as the $U$-value for the WMW test. For example, if we consider the above mentioned special case of a learning algorithm that always outputs the same predictions $\bm{f}$ independently of the training data, then the $U$-value between $\labelseq$ and $\bm{f}$ is equal to the number of LPOCV errors. However, if learning from the training data with revealed labels takes place, the LPOCV error distribution is not unique under the null hypothesis of uniform DCWLP. Namely, in addition to the number of data $\tsetsize$ and the weight $\posc$ of the label sequence, the shape of the LPOCV error distribution also depends on the combination of the learning algorithm and the data $\indset$. This is illustrated with the following example consisting of both a very simple learning algorithm and data:
\begin{example}\label{orderdirectionlearnerexample}
Assume that we have a sequence of 30 real-valued data $\indset\in\mathbb{R}^{30}$ of which 15 is labelled with 1 and the rest with zero.
Consider an algorithm that learns the following prediction function from the training data with revealed labels:
\[
\predfun(x)=
\left\{
\begin{array}{cc}
	x&\ureal>0.5\\
	0&\ureal=0.5\\
	-x&\ureal<0.5
\end{array}
\right.\;,
\]
where $\ureal$ is the AUC on the data with revealed labels with the values of $\indset$ used as predictions. Intuitively, the leaning algorithm simply determines whether the data labeled with 1 tend to have larger value of $x$ than those labeled with 0, or vice versa. The distribution of pairwise LPOCV errors obtained by uniformly sampling label permutations for this learning algorithm is depicted in Figure~\ref{figorderdirectionlearnerudist}.
\end{example}
\begin{figure}
	\centering
	\includegraphics[width=0.75\linewidth]{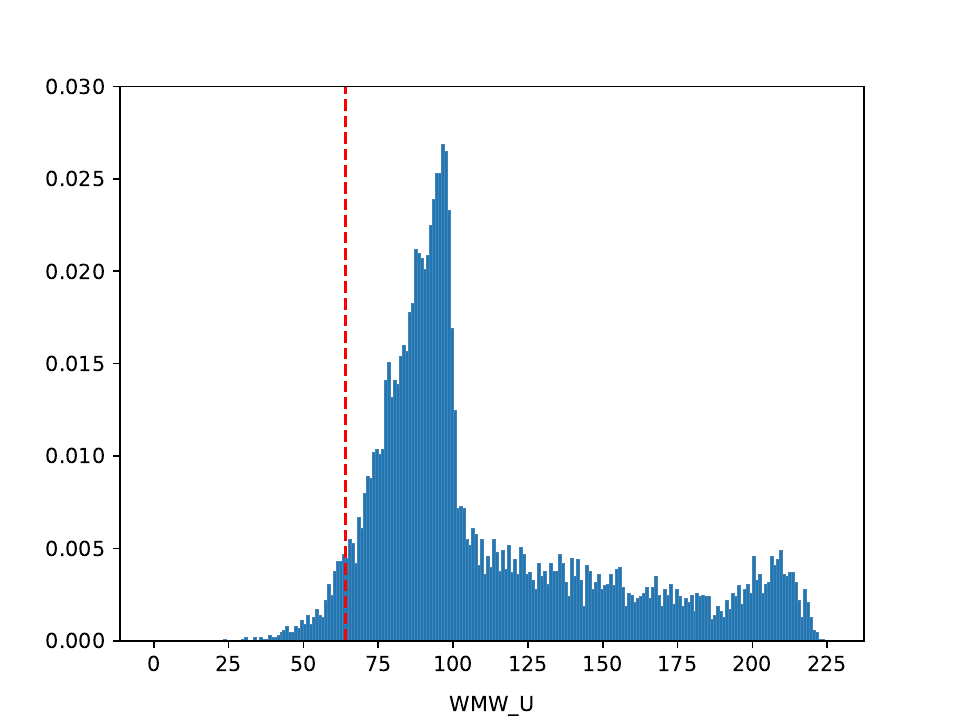}
	\caption{The distribution of U-values under the null hypothesis for a sample of size $\tsetsize=30$ and weight $\posc=15$ and with the order direction leaner. The red dashed line denotes the significance level 0.05, that is, 5\% of the probability mass is on the left side of the line.}
	\label{figorderdirectionlearnerudist}
\end{figure}
We observe that even the above presented learning algorithm that is able to learn only two alternative hypotheses is under uniform DCWLP associated with a LPOCV error distribution that is quite far from the bell-shaped Wilcoxon distribution one obtains from the WMW-test's $U$-values under null hypothesis. Moreover, in most cases, one does not have other means to determine the null $\uval$-distribution for each combination of data and a learning algorithm than brute force sampling. Nevertheless, through the above described link to error detecting codes, it is possible to determine upper bounds for the probability of obtaining at least as low LPOCV errors as the ones measured on the data under the null hypothesis of uniform DCWLP.

Recall from above that if the support of a DCWLP for certain $\tsetsize$ and $\posc$ forms a CWC, then there exists a learning algorithm with zero expected LPOCV error for that DCWLP. Accordingly, we show that the probability of obtaining a zero expected LPOCV error under the uniform DCWLP for any combination of a learning algorithm and data is upper bounded by the ratio of maximal size of CWC to the number of all constant weight words. Unfortunately, these maximal sizes of CWCs are only partly or approximately known. Indeed, according to \citet{sloane1989unsolved}, the two fundamental problems in coding theory are to find the maximal numbers of code words for binary or constant weight codes, with the constraints on word lengths, distances between code words and weights. Furthermore, these coding theory results do not as such provide any insight about the probabilities of different nonzero LPOCV errors under uniform DCWLPs. Therefore, to analyze the underlying combinatorics further, we introduce in this paper the concept of a \textbf{light constant weight code} (we follow the naming convention used in graph theory \citep{Asahiro2O15wlight}). We say that a set of constant weight words of length $\tsetsize$ and weight $\posc$ is \textbf{$\uu$-light}, if there exists a combination of a learning algorithm and data such that for any code word assigned as the labeling of the inputs, at most $\uu$ LPOCV errors take place. The maximal sizes of light constant weight codes form similar challenging combinatorial questions as those of the ordinary constant weight codes. As the last but not least contribution of this paper, we prove both upper and lower bounds for the maximal light constant weight code sizes. These bounds both provide new insights for the combinatorics of cross-validation and pave the way towards practical statistical tests on whether learning can take place regarding the problem under consideration.

This paper is organized as follows. In Section~\ref{prelimsection}, we formalize the considered topics and connect them to the relevant concepts in theoretical statistics. In Section~\ref{codingsection}, we carry out the combinatorial analysis of LPOCV with the machinery of coding and graph theory. The potential of the results for designing new statistical tests is experimentally simulated in Section~\ref{simusection} and the paper is concluded in Section~\ref{disc}.

\section{Preliminaries}\label{prelimsection}

In this section, we first provide the definitions of the statistical concepts used in the remaining of the paper in Section~\ref{basicdefsection}. We then continue in Section~\ref{independencesection} by defining what we call the null hypothesis of uniform DCWLP indicating that the class label assignments of data are drawn from a uniform distribution over constant weight words for a given length $\tsetsize$ and weight $\posc$.

\subsection{Notation and Basic Definitions}\label{basicdefsection}

Let $\cwwset(\tsetsize,\posc)=\{\binseq\in\{0,1\}^\tsetsize\mid\sum_{i=1}^\tsetsize\binseq_i=\posc\}$ for some $0<\posc<\tsetsize$. Since the elements of $\cwwset(\tsetsize,\posc)$ are binary vectors with exactly $\posc$ ones, we say that they are constant weight binary words of length $\tsetsize$ and weight $\posc$. For a word $\binseq$ and a permutation $\pi$ defined on the set $\{1,\ldots, \tsetsize\}$   we define the operation $\pi \cdot \binseq$ by 
\begin{equation}\label{perm}
   \pi \cdot \binseq=(\binseq_{\pi(1)},\binseq_{\pi(2)},\ldots, \binseq_{\pi(\tsetsize)}) 
\end{equation}
In particular, the transposition, denoted as $(i\;j)$, switches the $i$th and $j$th entry of the word and leaves  all other entries fixed.

Recall that  $\indset=(x_1,\ldots,x_\tsetsize)$ denotes an arbitrary sequence of inputs, that in this paper is considered fixed in advance. Therefore, we omit it from the forthcoming definitions and considerations below. From now on, $\labelvar$ denotes the categorical random variable corresponding to the binary constant weight learning problems, that is, its values are in $\cwwset(\tsetsize,\posc)$. By denoting $\labelvar\sim\textnormal{Uni}[\cwwset(\tsetsize,\posc)]$ or $\labelvar\sim\mathcal{D}(\cwwset(\tsetsize,\posc))$, we respectively indicate that the distribution of constant weight learning problems (DCWLP) associated to $\labelvar$ is either uniform or some other arbitrary  non-uniform one. For a particular realization of $\labelvar$, corresponding to the actual labels observed in an experiment, we use the symbol $\labelseq$. For alternative realizations, we also use the symbol $\binseq$.

Let $\lalg$ denote a learning algorithm. In this paper, we do not provide their explicit definitions, since they are not relevant in our considerations. Rather, we define them implicitly via their predictions of the hidden class labels, given a particular realization of the learning problem. In other words, the learning algorithms differ from each other only by their behavior with regard to the class label sequence $\labelseq\in\cwwset(\tsetsize,\posc)$ and which of its entries are revealed. In particular, as our focus is on leave-pair-out cross-validation, we only consider cases in which two entries of the label sequence are hidden, one entry labeled with 1 and other with 0.

Next, we formalize the procedure in which the whole $\labelseq$, except one datum labeled with 1 and another one labeled with 0, are used to predict which of the two points is which. Namely, we define the following concept that we refer to as leave-pair-out:
\begin{definition}[leave-pair-out (LPO)]\label{lpocvdef}
For any constant weight word $\binseq\in\cwwset(\tsetsize,\posc)$, and $1\leq i,j\leq\tsetsize$ such that $\binseq_i=1,\binseq_j=0$.
Let $\predfun_{\lalg(\binseq, i, j)}(i)$ and $\predfun_{\lalg(\binseq, i, j)}(j)$, respectively, denote the predictions made for the $i$th and $j$th data, when the training data provided for the learning algorithm $\lalg$ consists of the whole binary class label sequence $\binseq$, except that the $i$th and $j$th of its entries are hidden.
Then, for any deterministic and symmetric learning algorithm $\lalg$, we use the following short-hand notation:
\begin{align}\label{lpokernel}
\kernel_{\lalg}(\binseq,i,j)&\mapsto
\left\{
\begin{array}{ll}
0&\textnormal{If } \predfun_{\lalg(\binseq, i, j)}(i)>\predfun_{\lalg(\binseq, i, j)}(j)\\
0.5&\textnormal{If }\predfun_{\lalg(\binseq, i, j)}(i)=\predfun_{\lalg(\binseq, i, j)}(j) \\
1&\textnormal{If } \predfun_{\lalg(\binseq, i, j)}(i)<\predfun_{\lalg(\binseq, i, j)}(j)
\end{array}\right.
\end{align}
to indicate whether the prediction function inferred from the whole sample except a pair of indices with different class labels makes a mistake on the pair left out.
\end{definition}
Given this definition, the set of learning algorithms can be partitioned into equivalence classes according to their LPO behavior. Namely, we say that two learning algorithms $\lalg$ and $\lalg'$ are \textbf{$(\tsetsize,\posc)$-LPO equivalent} if
\begin{align}\label{lpoeqclass}
	\kernel_{\lalg}(\binseq,i,j)=\kernel_{\lalg'}(\binseq,i,j)
\end{align}
for all $\binseq\in\cwwset(\tsetsize,\posc)$ and $i,j$ such that $\binseq_i=1,\binseq_j=0$.

For any $\lalg$ and any $\binseq$ we have the following identity that directly follows from Definition~\ref{lpocvdef}:
\begin{align}\label{mainconstraint}
\kernel_{\lalg}(\binseq,i,j)=1-\kernel_{\lalg}((i\;j)\cdot\binseq,j,i)\;,
\end{align}
indicating that if the labels of the two differently labeled indices held out from the training set are transpositioned, the value of $\kernel$ is complemented.

As an important special case, where LPO is equivalent to the pairwise error in WMW statistic, we define what we call \textbf{constant learning algorithms}. For all $\binseq$,
\begin{align}\label{constantalgo}
\kernel_{\lalg_\predfun^{\textnormal{Constant}}}(\binseq,i,j)
=\left\{
\begin{array}{ll}
0&\textnormal{if }\predfun(i)>\predfun(j)\\
0.5&\textnormal{if }\predfun(i)=\predfun(j)\\
1&\textnormal{if }\predfun(i)<\predfun(j)
\end{array}\right.
\end{align}
for some real-valued prediction function $\predfun$ that depends neither on the binary sequence $\binseq$ nor the indices $i$ and $j$. That is, constant learning algorithms ignore the training set and always output the same prediction function $\predfun$ for all $i$, $j$ and $\binseq$. Clearly, (\ref{constantalgo}) is equal to the pairwise error (\ref{pairwiseerror}) mentioned in the context of the WMW statistic and test.

\subsection{Null Hypothesis of Uniform DCWLP}\label{independencesection}

In our forthcoming analysis, we adopt a \textbf{null hypothesis} that assumes the underlying distribution of the DCWLP to be uniform:
\begin{definition}[Null hypothesis]
\begin{align}\label{nullhypoeq}
\labelvar\sim\textnormal{Uni}[\cwwset(\tsetsize,\posc)]
\end{align}
\end{definition}
In this paper we focus on this setting, since it coincides exactly with the core assumption of the WMW-test (see e.g.  \citet{fay2010}).

We now consider how to measure the strength of evidence against the null hypothesis of uniform DCWLP.
For this purpose, we define what is known as leave-pair-out cross-validation (LPOCV) \citep{airola2011experimental}:
\begin{definition}[LPOCV]
For any $\lalg$ and any $\binseq\in\cwwset(\tsetsize,\posc)$:
\begin{align}\label{ustat}
\uval_{\lalg}(\binseq)
=
\sum_{i:\binseq_i=1}\sum_{j:\binseq_j=0}
\kernel_{\lalg}(\binseq,i,j)\;
\end{align}
\end{definition}
That is, the $\uval$-value counts how many LPO mistakes the learning algorithm $\lalg$ makes for the learning problem determined by $\binseq$.

Under the null hypothesis, the cumulative value distribution of $\uval_\lalg$ for a learning algorithm $\lalg$ can be expressed as:
\begin{align}\label{uexch}
\operatorname{P}_{\labelvar\sim\textnormal{Uni}(\cwwset(\tsetsize,\posc))}\left(\uval_\lalg(\labelvar)\leq \uu\right)&=\binom{\tsetsize}{\posc}^{-1}\sum_{\binseq\in\cwwset(\tsetsize,\posc)}\delta\left(\uval_\lalg(\binseq)\leq \uu\right)\;,
\end{align}
where $\uu$ denotes the argument of the cumulative distribution. The shape of the null distribution (\ref{uexch}) of LPOCV values depends on the learning algorithm $\lalg$. For example, the probability mass function corresponding to the learning algorithm presented in Example~\ref{orderdirectionlearnerexample} is depicted in Figure~\ref{figorderdirectionlearnerudist}, with the red dashed line determining the argument $\uu$.

As a special case, we consider the family of constant learning algorithms as in (\ref{constantalgo}) that output a fixed prediction functions $\predfun$ independently of the data. Then, (\ref{ustat}) becomes the $\uval$-value of the well-known Wilcoxon-Mann-Whitney test. In this case, the null distribution, sometimes referred in the literature as the Wilcoxon distribution, is also well-known. The exact cumulative probabilities for the Wilcoxon distribution can be obtained from the following recursion (see e.g. \citet{Cheung1997WMW} and references therein):
\begin{align}\label{wmwdistr}
\operatorname{P}_{\labelvar\sim\textnormal{Uni}(\cwwset(\tsetsize,\posc))}\left(\uval_\predfun(\labelvar)\leq \uu\right)
&=\binom{\tsetsize}{\posc}^{-1}Q(\uu, \tsetsize, \posc)\;,
\end{align}
where $\uu\in\mathbb{N}$ denotes the number of pairwise errors and $Q(\uu, \tsetsize, \posc)$
\[
=\left\{
\begin{array}{cc}
0     &  \textrm{ when }\uu<0\\
\binom{\tsetsize}{\posc}&\uu \geq \posc(\tsetsize-\posc)\\
\uu+1&\posc=1\textrm{ or }\tsetsize-\posc=1\\
Q(\uu, \tsetsize-1, \posc)+Q(\uu-\tsetsize+\posc, \tsetsize-1, \posc-1)     &\textrm{ otherwise } 
\end{array}
\right.\;.
\]
Knowing this cumulative distribution makes it possible to put the null hypothesis under a test by evaluating whether the observed U-value is too unlikely to be drawn from the distribution.

We next define some concepts related to constructing a test from the LPOCV-statistic. Following the terminology of \citet{casella2002statistical}, we say that a statistic  $p(\labelvar)$ is called a \textbf{valid $p$-value} if it, for any significance level $0\leq\alpha\leq 1$, fulfills
\begin{align}\label{pvaldef}
\operatorname{P}_{\labelvar\sim\textnormal{Uni}(\cwwset(\tsetsize,\posc))}(p(\labelvar)\leq\alpha)\leq\alpha
\end{align}
under the null hypothesis. That is, $p$-value is an upper bound on the test's \textbf{type I error}, the probability that its value is at least as small as $\alpha$ under the null hypothesis. Moreover, we say that $\critval(\alpha,\tsetsize,\posc)$ is a \textbf{critical value} for a significance level $\alpha$ if
\begin{align}\label{critvaldef}
\operatorname{P}_{\labelvar\sim\textnormal{Uni}(\cwwset(\tsetsize,\posc))}\big(\uval_\lalg(\labelvar)\leq \critval(\alpha,\tsetsize,\posc)\big)<\alpha
\end{align}
under the null hypothesis. Furthermore, given a significance level $\alpha$, learning algorithm $\lalg$ and some arbitrary categorical distribution $\mathcal{D}(\cwwset(\tsetsize,\posc))\neq\textnormal{Uni}(\cwwset(\tsetsize,\posc))$, the \textbf{power} function for the test is defined as:
\[
\operatorname{P}_{\labelvar\sim\mathcal{D}(\cwwset(\tsetsize,\posc))}(p(\labelvar)<\alpha)
\]
under the alternative hypothesis determined by $\mathcal{D}(\cwwset(\tsetsize,\posc))$. The power refers to the probability of rejecting an incorrect null hypothesis using the significance level $\alpha$. The probability of the so-called \textbf{type II error} under $\mathcal{D}(\cwwset(\tsetsize,\posc))$ indicating the failure to reject a wrong null hypothesis is simply one minus the power of the test. Note, however, that the power is always against a particular alternative distribution $\mathcal{D}$, that is usually unknown.

As an example, we again consider the WMW test. If one selects a significance level, say $\alpha=0.05$, the quantity (\ref{wmwdistr}) can be used as a p-value for a number of pairwise errors $\uu$ taking place in WMW U-statistic calculated for a sample. The largest numbers of errors $\uu$ for which one can reject the null for different values of $\posc$ and $\tsetsize-\posc$ are often presented in tables of critical values for certain standard significance levels. For the WMW-test, the values for $\alpha=0.05$ are shown in Figure~\ref{wmwcriticals}.

In contrast to the WMW test, the case for arbitrary learning algorithms is more complicated, as different learning algorithms may have differently shaped U-distributions (\ref{uexch}) under the null. In statistical literature, it is said that there is some unknown nuisance parameter that determines the shape of the distribution (\ref{uexch}). In our case, it can be interpreted so that the learning algorithm $\lalg$ (in particular its LPOCV behavior on the fixed sample $\indset$) is the nuisance parameter.

One of the simplest and most popular ways to tackle the nuisance parameter when designing new tests is the worst case analysis known as the ``supremum'' approach, which can be found from many textbooks on statistics. We rephrase the Theorem~8.3.27 by \citet{casella2002statistical} with our terminology and symbols:
\begin{theorem}\label{suptheorem}
Let $\uval_\lalg(\labelseq)$ as in (\ref{uexch}) be a test statistic such that its small values provide evidence against the null hypothesis. For each $\labelseq$, let
\[
p(\labelseq)=\sup_{\lalg}\operatorname{P}_{\binvar\sim\textnormal{Uni}(\cwwset(\tsetsize,\posc))}(\uval_{\lalg}(\binvar)\leq\uval_\lalg(\labelseq))\;.
\]
Then, $p(\labelvar)$ is a valid $p$-value.
\end{theorem}
While in the original form by \citet{casella2002statistical} the supremum is taken over all possible values of the unknown nuisance parameter determining the shape of the null distribution, here the supremum is simply taken over all possible learning algorithms. In the next section, we show how to compute the supremum over all possible deterministic and symmetric learning algorithms based on the theory of the classical error detecting codes.

\section{Combinatorial Analysis of the Null Distribution}\label{codingsection}

In Section~\ref{cweccsection}, we define light constant weight codes, an extension of the ordinary constant weight error correcting codes, and present the connection between the combinatorics of LPOCV and the considered codes. In Section~\ref{boudsection}, we prove an upper bound on the maximal size of light constant weight codes. In Section~\ref{bosesubsection}, we establish a lower bound, based  on an extension of a classical algebraic approach. 
\subsection{Constant Weight Error-Correcting Codes and the Johnson Graph}\label{cweccsection}

\begin{figure}
	\begin{adjustbox}{max width=\textwidth}
		\includegraphics[align=c]{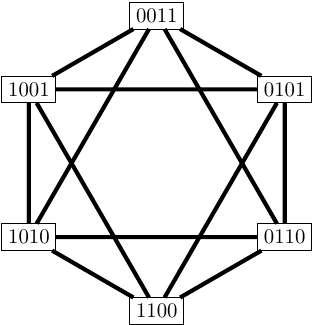}\hspace*{9pt}\includegraphics[align=c]{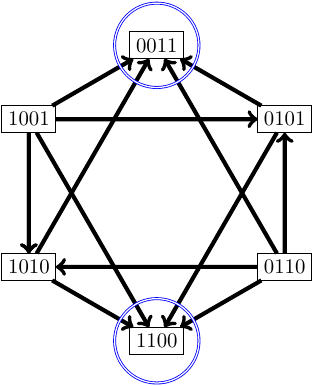}\hspace*{9pt}\includegraphics[align=c]{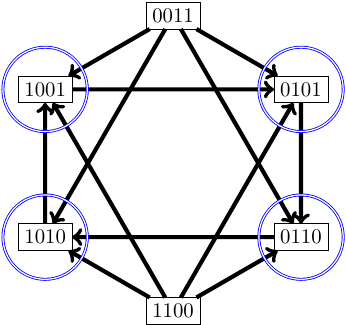}
	\end{adjustbox}
	\caption{Left: Illustration of an undirected Johnson graph $\johnson(4,2)$, whose vertices are labeled with the constant weight words of length $4$ and weight $2$, and two vertices are connected with an edge if the Hamming distance between the vertex labels is smaller than 4. Middle: The labels of a subset of the Johnson graph vertices (circled) corresponds to a constant weight code, since the vertices are disconnected from each other. Right: The labels of a subset of the Johnson graph vertices (circled) corresponds to a $1$-light constant weight code, since there exist an orientation (illustrated with arcs) under which the outdegrees of the vertices are at most $1$.}
	\label{graphexamples}
\end{figure}

Let us consider sub-sets of the constant weight words $\vecset\subset\cwwset(\tsetsize,\posc)$. If the Hamming distance between distinct vectors in $\vecset$ is at least $\hdist$, we say that $\vecset$ is a constant weight error correcting code of length $\tsetsize$, weight $\posc$ and minimum distance $\hdist$. With $\cwc(\tsetsize,\hdist,\posc)$, we denote the maximum number of binary vectors a constant weight code of length $\tsetsize$, Hamming distance $\hdist$ and weight $\posc$ can consist of. In this paper, we only consider minimum distance $\hdist=4$, and we call these shortly constant weight codes. These constant weight codes can also be considered via the Johnson graphs defined as follows:
\begin{definition}[Johnson graph]
A \textit{Johnson graph} $\johnson(\tsetsize,\posc)$ is an undirected graph, whose vertex set is $\cwwset(\tsetsize,\posc)$ and two vertices $\binseq$ and $\binseq'$ are connected with an edge if $(i\;j)\cdot\binseq=\binseq'$ for some $i$ and $j$ such that $\binseq_i=1$ and $\binseq_j=0$, that is, the Hamming distance between them is exactly $2$.
\end{definition}
An example of a Johnson graph $\johnson(4,2)$ is given in Figure~\ref{graphexamples}. The constant weight codes of minimum distance 4 correspond to the subsets of the Johnson graph vertices such that are all disconnected, and finding the value of $\cwc(\tsetsize,4,\posc)$ is equal to solving the maximum independent set problem on the Johnson graph $\johnson(\tsetsize,\posc)$. Figure~\ref{graphexamples} also contains an example of a constant weight code depicted as a subset of the Johnson graph $\johnson(4,2)$.

By an orientation $\Lambda$ of an undirected graph, we refer to an assignment of a direction to each one of its edges. An oriented graph can be conveniently represented with its adjacency matrix whose rows and columns are both indexed with the vertices and the element at position $(i,j)$ is one when there is an arc from vertex i to vertex j, and zero otherwise. The number of arcs starting from the vertex $\binseq$, that is, the number of positive entries on the corresponding row of the adjacency matrix, is the outdegree of the vertex under the orientation $\Lambda$, which we denote as $\degree^+_\Lambda(\binseq)$.

Adopting the term introduced by \citet{Asahiro2O15wlight}, we say that a vertex $\binseq$ of the Johnson graph $\johnson(\tsetsize,\posc)$ is $\uu$-light under orientation $\Lambda$ if $\degree^+_\Lambda(\binseq)\leq\uu$. Moreover, let
$\ucount(\uu,\tsetsize,\posc,\Lambda)$ denote the number of $\uu$-light vertices in the Johnson graph $\johnson(\tsetsize,\posc)$ under its orientation $\Lambda$.

Following this naming convention, we extend the concept of constant weight codes to $\uu$-light codes:
\begin{definition}[$\uu$-light codes]
We call a set $\vecset\subset\cwwset(\tsetsize,\posc)$ a $\uu$-light  $(\tsetsize,\posc)$ constant weight code, shortly a $\uu$-light code, if there exists such an orientation of the Johnson graph $\johnson(\tsetsize,\posc)$ under which all of the vertices corresponding to the elements of $\vecset$ are $\uu$-light.
\end{definition}

Finding a specific orientation fulfilling certain constraints for a given graph is called a graph orientation problem. \citet{Asahiro2O15wlight} introduced a graph orientation problem called ``maximize $\uu$-light'', that refers to finding, for a given graph, the maximum possible number of $\uu$-light vertices. We are in this paper mostly focusing on the orientations of the Johnson graphs and in particular we define the following concept:
\begin{definition}[Maximal size $\uu$ light code]
For binary vectors of length $\tsetsize$ and weight $\posc$, we denote the largest possible size of $\uu$ light code as
\[
\maxulight(\uu,\tsetsize,\posc)=\max_\Lambda\ucount(\uu,\tsetsize,\posc,\Lambda)\;,
\]
where the maximum is taken over the orientations of the Johnson graph $\johnson(\tsetsize,\posc)$.
\end{definition}
Accordingly, finding the largest possible constant weight code with minimum distance $4$ is equal to solving the maximize $0$-light problem on the corresponding Johnson graph. This is a well-known NP-hard problem in coding theory, but for which there exists good upper bounds (see e.g. \citep{ostergard2010cwc}). Similarly, finding the largest possible $\uu$-light constant weight codes is equal to solving the maximize $\uu$-light graph orientation problems on the Johnson graphs. We get back to these below in Section~\ref{boudsection}. An example of a $1$-light code depicted as a subset of the Johnson graph $\johnson(4,2)$ is shown Figure~\ref{graphexamples}.

Now we have the machinery for characterizing the possible outcomes of LPOCV.
\begin{proposition}\label{thelink}
There exists a bijection between the set of orientations of a Johnson graph $\johnson(\tsetsize,\posc)$ and the set of $(\tsetsize,\posc)$-LPO equivalence classes of learning algorithms.
\end{proposition}
\begin{proof}
The proof is straightforward as it only requires connecting the two points of view with each other. Let us consider an algorithm $\lalg$, a representative of a single $(\tsetsize,\posc)$-LPO equivalence class. The constraint (\ref{mainconstraint}) determines the direction of an arc between the vertices $\binseq$ and $(i\;j)\cdot\binseq$ in the Johnson graph $\johnson(\tsetsize,\posc)$, so that the arc points from $\binseq$ to $(i\;j)\cdot\binseq$ if $\lalg$ makes and erroneous LPO prediction for the pair $(i,j)$ with the labeling $\binseq$ and a correct one with $(i\;j)\cdot\binseq$. One LPO equivalence class by its definition determines all values of $\kernel_{\lalg}(\binseq,i,j)$, and since the Johnson graph's edges are between vertices differing only by a single transposition, the class is mapped to a Johnson graph orientation. If two equivalence classes differ for some $\binseq\in\cwwset(\tsetsize,\posc)$ and index pair $(i,j)$, then also the orientation of the corresponding arc differs, and hence the mapping from the classes to the orientations is ``one-to-one''. Moreover, the arc directions of an oriented Johnson graph determine the values of $\kernel_{\lalg}(\binseq,i,j)$ for all $\binseq\in\cwwset(\tsetsize,\posc)$ and index pairs $(i,j)$ according to the constraint  (\ref{mainconstraint}), and hence the mapping is also ``onto''.
\end{proof}

The above result indicates that we can use the properties of the Johnson graphs, particularly their orientations, to characterize the LPOCV behavior of learning algorithms. We denote by $ \Lambda_{\lalg}$ the orientation associated with learning algorithm $\lalg$. Then, $\ucount(\uu,\tsetsize,\posc,\Lambda_{\lalg})$ is the number of weight $\posc$ learning problems for which the number of LPOCV errors made by $\lalg$ is at most $\uu$. Moreover, we can adopt the other above introduced coding theory concepts for the analysis as illustrated by the following corollary:
\begin{corollary}\label{maxcapacitycorollary}
The maximum number of different label sequences of length $\tsetsize$ and weight $\posc$ for which any fixed leaning algorithm makes at most $\uu$ LPO errors in LPOCV is
\[
\max_{\lalg}\left\arrowvert\left\{\binseq\in\cwwset(\tsetsize,\posc)\mid \uval_\lalg(\binseq)\leq 
\uu\right\}\right\arrowvert=\maxulight(\uu,\tsetsize,\posc)\;.
\]
\end{corollary}
Finally, to quantify the supremum over the learning algorithms in Theorem~\ref{suptheorem}, the following corollary summarizes that it coincides with the maximal size of a $\uu$-light $(\tsetsize,\posc)$-code.
\begin{corollary}
Fix $\tsetsize$ and $\posc$. Then, for any $\labelseq$,
\[
\sup_{\lalg}\operatorname{P}_{\binvar\sim\textnormal{Uni}(\cwwset(\tsetsize,\posc))}(\uval_{\lalg}(\binvar)\leq\uval_\lalg(\labelseq))=\binom{\tsetsize}{\posc}^{-1}\maxulight\left(\uval_\lalg(\labelseq),\tsetsize,\posc\right)\;.
\]
\end{corollary}
Next, we quantify the values of $\maxulight(\uu,\tsetsize,\posc)$ by generalizing some classical results on constant weight error correcting codes to $\uu$-light codes.

\subsection{Bounds for the Maximal LPO Capacity}\label{boudsection}

\begin{figure}
	\centering
	\includegraphics[width=0.75\linewidth]{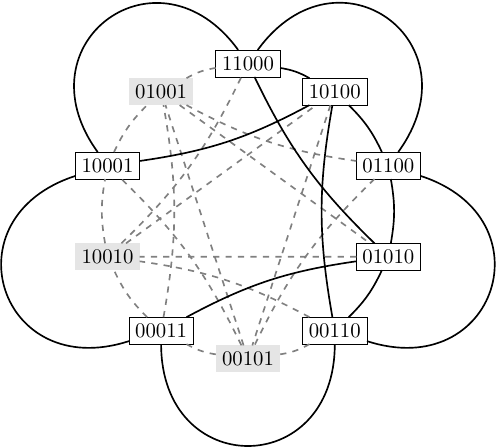}
	\caption{An illustration of a Johnson graph $\johnson(5,2)$ and its sub-graph consisting only of vertices with maximum degree 4.}
	\label{2lightsubgraph}
\end{figure}

We start by presenting some values of $\maxulight(\uu,\tsetsize,\posc)$, that are easy to calculate. We will use the following facts and terminology:

For any undirected finite simple graph $G$, for each vertex $x$ let us call \textbf{$e$-degree} the degree of $x$ if it is even, and the degree of $x$ plus $1$ if the degree is odd. Denote this e-degree by $e(x,G)$.

\begin{fct}\label{zerofact} Every undirected finite simple graph $G$ has an orientation in which the outdegree of each vertex $x$ is at most $e(x,G)/2$. 
\end{fct}

An orientation of $G$ is called a \textbf{balanced orientation}, in which orientation the outdegree of each vertex $x$ is at most $e(x,G)/2$.

\begin{proof}By induction on the the number of edges of $G$. If $G$ is a forest, then orienting the edges of each tree component towards a root vertex of the tree we have the desired orientation.
	If $G$ is not a forest, let $H$ be the graph obtained from $G$ by removing the edges of some circuit. For each vertex $x$ of the circuit $e(x,H) = e(x,G) -2$. Orient the edges on the circuit according to one of the two circular traversals of the circuit, and take any balanced orientation of $H$, which exists by the induction hypothesis. This defines a balanced orientation of $G$.
\end{proof} 
The following fact is a straightforward consequence of Fact \ref{zerofact}:

\begin{fct}\label{fact}
	If $G$ is an undirected graph with maximum degree $\Delta(G)\leq 2\uu$, then there exists an orientation of $G$ in which every vertex of $G$ is $\uu$-light. 
\end{fct} 
In the literature, these are known as Eulerian orientations if every vertex in $G$ has an even degree.

\begin{fct}\label{f2} A $\uu$-light set $S$ of vertices of the complete graph $K$ cannot have more than $2\uu+1$ elements. 
\end{fct}

\begin{proof}Let $S$ have $m$ elements, and fix an orientation of the complete graph in which every element of $S$ has outdegree at most $\uu$. Then the $m(m-1)/2$ edges with both end vertices in $S$ 
are over-counted (or counted exactly) by adding up the outdegrees of the members of $S$, which cannot exceed $m\uu$ due to $\uu$-lightness. Thus $m\uu$ is at least $m(m-1)/2$, i.e. $m$ is at most $2\uu+1$. 
\end{proof}

An example of a subgraph of a Johnson graph $\johnson(5,2)$ whose vertices have  maximum degree $4$, is shown in Figure~\ref{2lightsubgraph}. According to Fact~\ref{fact}, there exists a $2$-light orientation of the subgraph.

We observe that, for $\posc=1$ or $\tsetsize-\posc=1$, the Johnson graph is a tournament graph for which the values of $\maxulight$ are already known (see e.g. \citet{Mousley2000thesis}). We give here a direct proof of the following known result.
\begin{proposition}
	\label{boundary0}
	For $\posc=1$ or $\tsetsize-\posc=1$, we have
	\[
	\maxulight(\uu,\tsetsize,1)=\maxulight(\uu,\tsetsize,\tsetsize-1)=\min\left(2\uu + 1,\tsetsize
	\right)\;
	\]
\end{proposition}

\begin{proof}The Johnson graph in this case is a complete graph $K_n$ on $n$ vertices, regular of degree $n-1$. Fix any value for $\uu$.
	Denote $L(\uu,n,1) = L(\uu,n,n-1)$ simply by $L(\uu)$.
	By Fact \ref{f2}, $L(\uu)$ is at most $2\uu+1$. Trivially also $L(\uu)$ is at most $n$.
	\\
	Case 1:  If $n-1$ is at most $2\uu$, which is exactly the hypothesis of Fact \ref{fact}, then by Fact \ref{fact} the complete graph $K_n$ has an orientation (tournament) in which all outdegrees are at most $\uu$, thus $L(\uu)$ is $n = min (2\uu+1, n)$.
	\\
	Case 2: If $n-1$ is at least $2\uu$, i.e. $n$ is at least $2\uu+1$, enumerate some subset $S$ of $2\uu+1$ vertices of $K_n$ as $x(1), x(2),\ldots,x(2\uu), x(2\uu+1)$, the indices being integers modulo $2\uu+1$.
	For any edge $r$ of $K_n$ with end vertices $y\notin S$ and $z\in S$, orient $r$ from $y$ to $z$. If both $y$ and $z$ are outside $S$, orient $r$ arbitrarily.
	If $y=x(i)$ and $z=x(j)$, then orient $r$ from $y$ to $z$ if $j-i$ is congruent to one of the first $W$ integers $1,\ldots,W$.
	In this orientation, all vertices of $S$ are $W$-light, thus $L(\uu)= 2\uu+1 = \min (2\uu+1, n)$.
\end{proof}

By the \textbf{average degree} of a graph with $n$ vertices we mean $(d_1+\cdots +d_n)/n$, where $d_1,\ldots,d_n$ is the degree sequence.

\begin{lemma}\label{avgdegreelemma}
If the line graph of a graph G has an orientation in which all outdegrees are at most a number $\uu$, then the average degree of $G$ is at most $\uu+1$.
	
\end{lemma}

\begin{proof} Let $n$ and $m$ be the number of vertices and edges of $G$, respectively. For the degree sequence we have  $d_1 +\cdots + d_n  = 2m$. 
	Let $D_1,\ldots ,D_m$ be the sequence of outdegrees in an orientation of the line graph of $G$, each $D_i$ being at most $\uu$.
	The sum of the $D_i$'s is obviously equal to the number of edges of the line graph $H$ of $G$.
	Each edge of $H$ consists of two distinct edges of $G$ incident with a unique common vertex of $G$.
	Thus the edges of $H$ can be classified into $n$ classes, and these classes have cardinalities $\binom{d_1}{2},\ldots ,\binom{d_n}{2}$. 
	The sum  $\binom{d_1}{2}+\cdots +\binom{d_n}{2}$ is thus equal to the sum $D_1+\cdots +D_m$, which is at most $m\uu$.
	Therefore   $\binom{d_1}{2}+\cdots +\binom{d_n}{2}\leq m\uu=  (d_1 +\cdots + d_n)\uu / 2$.
	This can be easily rewritten as 
	\begin{equation}\label{eqA}
		d_1^2+\cdots +d_n^2\leq (d_1+\cdots +d_n)(\uu+1)
	\end{equation}
	Now apply the Cauchy-Schwarz inequality to the two $n$-vectors $(d_1,\ldots,d_n)$  and  $(1,\ldots,1)$. We get  $(d_1+\cdots+d_n)^2\leq (d_1^2+\cdots +d_n^2) n.$
	Combining with  inequality \ref{eqA} we have
	$$(d_1+\cdots+d_n)^2\leq (d_1+\cdots +d_n) n(\uu+1).$$
	Dividing by $(d_1+\cdots +d_n) n$ we get the intended result.
\end{proof}

For the cases $\posc=2$ or $\tsetsize-\posc=2$, we have the following result:
\begin{proposition}
	\label{boundary2}
	\[
	\maxulight(\uu,\tsetsize,2)=\maxulight(\uu,\tsetsize,\tsetsize-2)=
	\min\left(\left\lfloor\frac{(\uu+1)\tsetsize}{2}\right\rfloor,\binom{\tsetsize}{2}\right)\;
	\]
\end{proposition}

\begin{proof}
	Under hypotheses $m= \left\lfloor n(W+1)/2\right\rfloor$,  $W+1<n-1$,
	$G(n,m,\uu)$ will be defined as a spanning subgraph of the complete graph $K_n$ on the group $\mathbb{Z}_n$ of integers $0,\ldots,n-1$ with addition modulo $n$.
	The orientation of the line graph of $G(n,m,\uu)$ will be defined by an orientation of the line graph of $K_n$. 
	For this purpose we need an auxiliary directed graph $C_n$ on $\mathbb{Z}_n$, in which the arrows are the couples $(i, i+1)$, addition understood mod $n$ (this graph is a directed cycle). 
	Now in the line graph of $K_n$, -- where the vertices are couples ${u,w}$ of elements of $\mathbb{Z}_n$ --, we define an arrow going from $\{x,y\}$ to $\{y,z\}$,  -- assuming $x,y,z$ are three distinct elements of $\mathbb{Z}_n$, -- whenever $y$ is on the unique shortest directed path from $x$ to $z$.
	
	The definition of $G(n,m,\uu)$ depends on the parity of $n$ and $\uu$, according to which we give the definition in three cases:
	
	Case 1. If $\uu$ is odd, then the edge set of $G(n,m,W)$ consists of those couples ${x,y}$ for which $x-y$ or $y-x$ is congruent to one of the $(\uu+1)/2$ integers $1,\ldots,(\uu+1)/2$ modulo $n$.
	
	Case 2. If $\uu$ is even and $n$ is even, then the edge set of $G(n,m,\uu)$ consists of those couples ${x,y}$ for which $x-y$ or $y-x$ is congruent to one of the $\uu/2$ integers $1,\ldots,\uu/2$ modulo $n$ or $y$ is congruent to $x+n/2$.
	
	Case 3. If $\uu$ is even and $n$ is odd, then the edge set of $G(n,m,\uu)$ consists of those couples ${x,y}$ for which $x-y$ or $y-x$ is congruent to one of the $\uu/2$ integers $1,\ldots,\uu/2$ modulo $n$  or $x$ is congruent to one of the integers $1,\ldots,(n-1)/2$ and $y$ is congruent to  $x+(n-1)/2$.
	
	This proves that for any $\uu$ the graph $J(n,2)$ can be so oriented, that $\left\lfloor{(\uu+1)\tsetsize/2}\right\rfloor$ vertices of $J(n,2)$ will have outdegree at most $\uu$.
	
	It remains to see that $J(n,2)$ has no orientation in which $k > (\uu+1)n/2$ vertices have outdegree at most $\uu$. Indeed, the set of such vertices constitute the edge set of a spanning subgraph $G$ of the complete graph $K_n$ on vertices. The number $k$ of edges of $G$ is $n/2$ multiplied by the average degree of $G$, implying that the average degree exceeds $\uu+1$, which is impossible according to the Lemma~\ref{avgdegreelemma}. 
\end{proof}

Since the largest number of pairwise non-incident edges of the complete graph $K_n$ on $n$ vertices is the integer part of $n/2$, Proposition~\ref{boundary2} is obvious for $\uu=0$, coinciding with the obvious case for $w=2$ of a result referred to as the Johnson bound. Surprisingly, Johnson's original proof technique (\citet{Johnson1962bound}) also gives inequalities formally analogous to the Johnson bound that we shall establish in the next Proposition. This generalization from $\uu=0$ to positive $\uu$ does not in any way tighten the Johnson bound for constant weight codes (which to our knowledge has only been sharpened  under particular conditions, restricted to special cases, see e.g. in \cite{agrell2000upper,etzion2014new,le2019simple}). Thus the following inequalities are, and expected to remain, unsharp in the same sense that the Johnson bounds are only estimates.

\begin{proposition}
	\begin{align}\label{jb1}
		\maxulight(\uu,\tsetsize,\posc)&\leq\left\lfloor\frac{\tsetsize}{\posc}\maxulight(\uu,\tsetsize-1,\posc-1)\right\rfloor\\
		\label{jb2}
		\maxulight(\uu,\tsetsize,\posc)&\leq\left\lfloor\frac{\tsetsize}{\tsetsize-\posc}\maxulight(\uu,\tsetsize-1,\posc)\right\rfloor
	\end{align}
\end{proposition}
\begin{proof}
	Let us form a $\maxulight\times\tsetsize$-matrix $M$ whose rows are the binary vectors corresponding to the $\uu$-light vertices of an oriented Johnson graph $\johnson(\tsetsize, \posc)$ with $\maxulight=\maxulight(\uu,\tsetsize,\posc)$ $\uu$-light vertices under an orientation $\Lambda$. Next, let us consider the column that has the largest number of ones on $M$ and let $\psi$ be the number of ones on this column. Then, let us delete this column and all such rows that have a zero in this column from $M$. Now, the remaining $\psi\times(\tsetsize-1)$ sub-matrix contains a sub-set of vertices of a Johnson graph $\johnson(\tsetsize-1, \posc-1)$, a sub-graph of $\johnson(\tsetsize, \posc)$. The vertices listed in this sub-matrix are $\uu$-light also on $\johnson(\tsetsize-1, \posc-1)$ under the orientation $\Lambda$, and hence $\psi\leq \maxulight(\uu,\tsetsize-1,\posc-1)$.

	Each of the $\maxulight$ rows has weight $\posc$, so the overall number of ones in $M$ is $\maxulight\posc$. But each column contains at most $\psi$ ones, so $\maxulight\posc\leq\tsetsize\psi$. Therefore $\maxulight\posc\leq\tsetsize \maxulight(\uu,\tsetsize-1,\posc-1)$, which proves (\ref{jb1}). The bound (\ref{jb2}) is obtained via considering the zeros instead of ones.
\end{proof}

Inequality (\ref{jb1}) should be viewed as a rate-of-increase bound the sequence of values $L(W,n,w)$ obtained when $W$ is fixed and $n$ and $w$ are varied by increasing both simultaneously by $1$ at each step. Similarly, inequality (\ref{jb2}) should be viewed as a rate-of-increase bound on the sequence $L(\uu,n,w)$ indexed by $n$ alone, both $W$ and $w$ being fixed. One can obtain an upper bound for each particular $L(\uu,n,w)$ by repeatedly applying (\ref{jb1}) and (\ref{jb2})  and using the already known exact values, such as the ones presented in Proposition~\ref{boundary2}.

\begin{remark}
	The bound provided by inequalities (\ref{jb1})-(\ref{jb2}) is not tight. For example, for $w=2$ inequality (\ref{jb1}) means that the ratio of increase  $L(\uu, n, 2) /  L(\uu, n-1, 1)$ is bounded by $n/2$, while for odd $W$ the exact value of this ratio is $$n(\uu+1) / 2 (2\uu+1) = (n/2 ) (\uu+1)/(2W+1)$$ 
	For $w=2$ inequality (\ref{jb2}) means, for odd $W$, that the ratio  $L(W, n, 2) /  L(W, n-1, 2)$  is bounded by $n/(n-2)$, while the exact ratio is $n/(n-1)$ by Proposition \ref{boundary2}, which bounds are asymptotically equivalent as $n$ tends to infinity. We do not know how tight is (\ref{jb2}) for higher values of $w$, as we do not have exact expressions for $L(W,n,w)$ in that case.
\end{remark}

Next we shall generalize to light codes another classical result appearing in  \citet{Graham1980lower} on lower bounds for distance $4$ constant weight code sizes.

Recall that for any underlying set $U$ of $n$ elements and positive integer $w$ the Johnson graph $J(n,w)$ can be represented as a graph that has as vertex set the $w$-subsets of $U$, two of these being adjacent in $J(n,w)$ whenever their symmetric difference has two elements. Clearly whatever the set $U$ of cardinality $n$ is, the structure of the graph $J(n,w)$ will be the same. However, what the set $U$ is matters in the sense that an algebraic structure on $U$ may provide the means of defining large independent or $\uu$-light sets of vertices of the Johnson graph. Elementary group theory can offer such structures.

Corresponding to the case $\uu=0$, the following seminal method is apparent in  \citet{bose1978codes} and also in  \cite{Graham1980lower}. Take $U$ to be the cyclic group $\mathbb{Z}_n$, and consider the function $\tau$ associating to each $w$-subset of $U=\mathbb{Z}_n$ the sum of its elements in $\mathbb{Z}_n$. For each $j \in \mathbb{Z}_n$ the set of those $w$-subsets that are mapped by $\tau$ to $j$ is independent in $J(n,w)$, and the largest of these has necessarily at least $\binom{n}{w}/n$ elements. Neither this method nor its extensions are guaranteed to yield the largest independent or $\uu$-light sets in $J(n,w)$.

It was already pointed out by  \citet{mceliece1980constantinrao}, and also by \citet{klove1981bound} that instead of $\mathbb{Z}_n$ taking $U$ to be another, non-cyclic commutative group $G$ of order $n$ may yield larger independent sets in $J(n,w)$. The following example already shows this. For $n=8, w=3$  and taking for $U$ the direct product $\mathbb{Z}_2\times \mathbb{Z}_4$, there are only 4 (four) $3$-subsets of $G$ whose sum is $(1,1)$. Therefore for some $(a,b) \in \mathbb{Z}_2 \times\mathbb{Z}_4$ there must be more than $\binom{8}{3}/8=7$  of those $3$-subsets whose sum is $(a,b)$. However, for any $j \in \mathbb{Z}_8$ the number of those $3$-subsets of $\mathbb{Z}_8$ whose sum is $j$ is exactly $7$. 

\subsection{Extended Bose-Rao construction}\label{bosesubsection}

For obtaining large $\uu$-light sets in $J(n,w)$ the method appearing in in  \citet{bose1978codes} and  in  \cite{Graham1980lower} can be further extended as follows, under the assumption that $\uu$ is at most $n/4$. Since $J(n,w)$ and $J(n, n-w)$ are always isomorphic, we assume that $w$ is at most $n/2$. 

For a fixed $\uu$ not exceeding $n/4$ the extended method consists of constructing a set $S$ of vertices in $J(n,w)$  inducing a subgraph in which all vertices have degree at most $d=2\uu$, and which set is therefore easily seen to be $W$-light. (Because we can add edges to the subgraph induced by $S$ to obtain an Eulerian graph $H$ that we can orient according to an Euler tour.)  Let $m=n-d$. 

Let $G$ be a commutative group, finite or infinite but of cardinality at least $2m$, having a congruence $C$ such that the quotient $G/C$ is of finite cardinality $m=n-d$. 
For example, $G$ can be the integers $\mathbb{Z}$ and $C$ the congruence $\mod m$ on $\mathbb{Z}$. Or $G$ can be the direct product $\mathbb{Z}_m \times \mathbb{Z}_2$ with the congruence induced by the  projection to the first factor $\mathbb{Z}_m$.

Let $N$ be a set of $n$ elements of $G$ meeting each congruence class in either $1$ or $2$ elements. Obviously $N$ meets $d$ classes in $2$ elements, the other classes in $1$ element.
We define a \textbf{link} as a $2$-subset of $N$ contained in a congruence class. The number of links is $d$.

Consider the Johnson graph $J(n,w)$ whose underlying set is $N$. Let $V$ denote the set of  $w$-subsets of $N$. 
Consider the  map $\tau$ associating to each $X \in V$ the $C$-class of the sum in $G$ of the elements of $X$.
If $X$ is any member of the reverse image $K$ under $\tau$ of any member of $G/C$, then in the Johnson graph under consideration the neighbors of $X$ within $K$ are precisely the symmetric sums of $X$ with those links that intersect $X$ in a singleton.
One of these reverse image sets must have cardinality at least $\binom{n}{w}/m=\binom{n}{w}/(n-d)=\binom{n}{w}/(n-2\uu)$. Consequently we have:

\begin{proposition}\label{lowerbounds}
	\[
	\maxulight(\uu,\tsetsize, \posc)\geq\frac{1}{\tsetsize-2\uu}\dbinom{\tsetsize}{\posc}
	\textnormal{ if }\tsetsize\geq 4\uu
	\]\hfill $\blacksquare$
\end{proposition}

The lower bound  $\binom{n}{w}/(n- 2\uu)$   on the size of the largest $W$-light set of vertices in $J(n,w)$ is not sharp, in several senses, as shown in the case $n=8, w=2, \uu=2$. 

First, in this case the lower bound  $\binom{n}{w}/(n- 2\uu)=7$.  However, using the integers $\mathbb{Z}$ as $G$ and letting $C$ be the congruence mod $4$, letting $N$ to be the set of positive integers $1,\ldots,8$, and examining the map $\tau$ associating to each two-element subset of $N$ their sum in $G/C$ (isomorphic to $\mathbb{Z}_4$), the reverse image of the $C$-class of $1$ will contain $8$ pairs, which constitute a $2$-light set of vertices in $J(8,2)$.

Second, it is not difficult to see that the size $8$ of the resulting $2$-light set in $J(8,2)$ cannot be improved by the same method if we choose $\mathbb{Z} \times \mathbb{Z}$ as $G$ and $C$ such that $G/C$ is the other group of order $4$, the Klein group $\mathbb{Z}_2 \times \mathbb{Z}_2$.

Third, a $2$-light set of $12$ vertices can be constructed in $J(8,2)$ as follows. Consider the graph of the $3$-dimensional cube $K$ with 8 vertices and $12$ edges. Each edge defines a $2$-subset of the set $U$ of the vertices of the cube. To each vertex $v$ of $K$ associate the set of the $3$ edges $\{a,b,c\}$  incident with $v$, denote this set by $E(v)$. It spans in $J(8,2)$ a triangle, on which we can place a circular orientation. Each member of $E(v)$ will have a unique successor in this circular order (a small regular tournament). An edge of the subgraph of $J(8,2)$ spanned by $12$ edges of $K$ is defined by a pair $\{e,b\}$  of incident edges belonging to a unique set $E(v)$. In $J(8,3)$ orient the edge $\{e,b\}$  from $e$ to $b$ if $b$ is the successor of $e$ in this $E(v)$. This can be completed to an orientation of $J(8,3)$ in which all the $12$ vertices corresponding to edges of the cube have outdegree at most $2$.

\section{Simulations}\label{simusection}

In our analysis, we have characterized some constraints for the shape of null distributions of learning algorithms under the null hypothesis (\ref{nullhypoeq}). These constraints, together with the supremum approach of Theorem~\ref{suptheorem}, provides us possibilities for designing statistical tests for this hypothesis. However, for the sake of discussion, we first go through certain interesting special cases of learning algorithm families.

Firstly, we recollect the class of constant learning algorithms (\ref{constantalgo}) whose null distribution coincides with that of the Wilcoxon distribution. We also note that if the learning algorithm is stable in the sense small changes in the training set do not affect to the inferred prediction function too much, the null distribution is likely to be the same as that of the standard WMW-U test. As the exact shape of the Wilcoxon distribution (depicted in Figure~\ref{figureofwilcoxondistribution}) is known and can be calculated via the recursion (\ref{wmwdistr}), we can construct statistical tables consisting of the critical values as in (\ref{critvaldef}), as illustrated for $\alpha=0.05$ in Figure~\ref{wmwcriticals}.

Secondly, if we consider the LPO-equivalence classes of learning algorithms as in (\ref{lpoeqclass}) that correspond to the orientations of the Johnson graphs as stated in Proposition~\ref{thelink}, it is easy to conclude that the most of the equivalence classes have the null distribution concentrated around 0.5 much more heavily that of the Wilcoxon distribution. For example, we may consider what we call a random learning algorithm that uses the training set as a random seed to infer prediction for the two held out points. This corresponds to drawing a random orientation for the Johnson graph corresponding to the LPO behavior of the learning algorithm. However, while most of the equivalence classes do have this type of a low variance null distribution, most of the learning algorithms are completely useless in practice, and hence this does not help us in designing a practical test. These are, nevertheless, interesting in the sense that in their Johnson graph there would be no vertices with zero or even a reasonably low number of pairwise errors, so obtaining a good result is not possible even in principle. This is elaborated more in Section~\ref{disc}.

Thirdly, if the learning algorithm coincides with a maximal size constant weight code or a maximal size light constant weight code with a specific lightness parameter $\uu$, the null distribution can be very extreme. However, it is very unlikely that a practical learning algorithm would coincide with a maximal code unless a malicious adversary has especially designed it to have high changes to produce a LPOCV results of exactly $\uu$ pairwise errors under the independence hypothesis. For example, if we consider the largest $\uu$-light codes that can be constructed with the method provided by Proposition~\ref{lowerbounds}, we get a table of critical values as illustrated in Figure~\ref{grahansloanecriticals}. We observe that a test based on the table would be far more conservative than WMW test, and hence it may be impractical in real-world studies. 

\subsection{Critical values based on empirical estimates}

\begin{figure}
\centering
\includegraphics[width=0.9\linewidth]{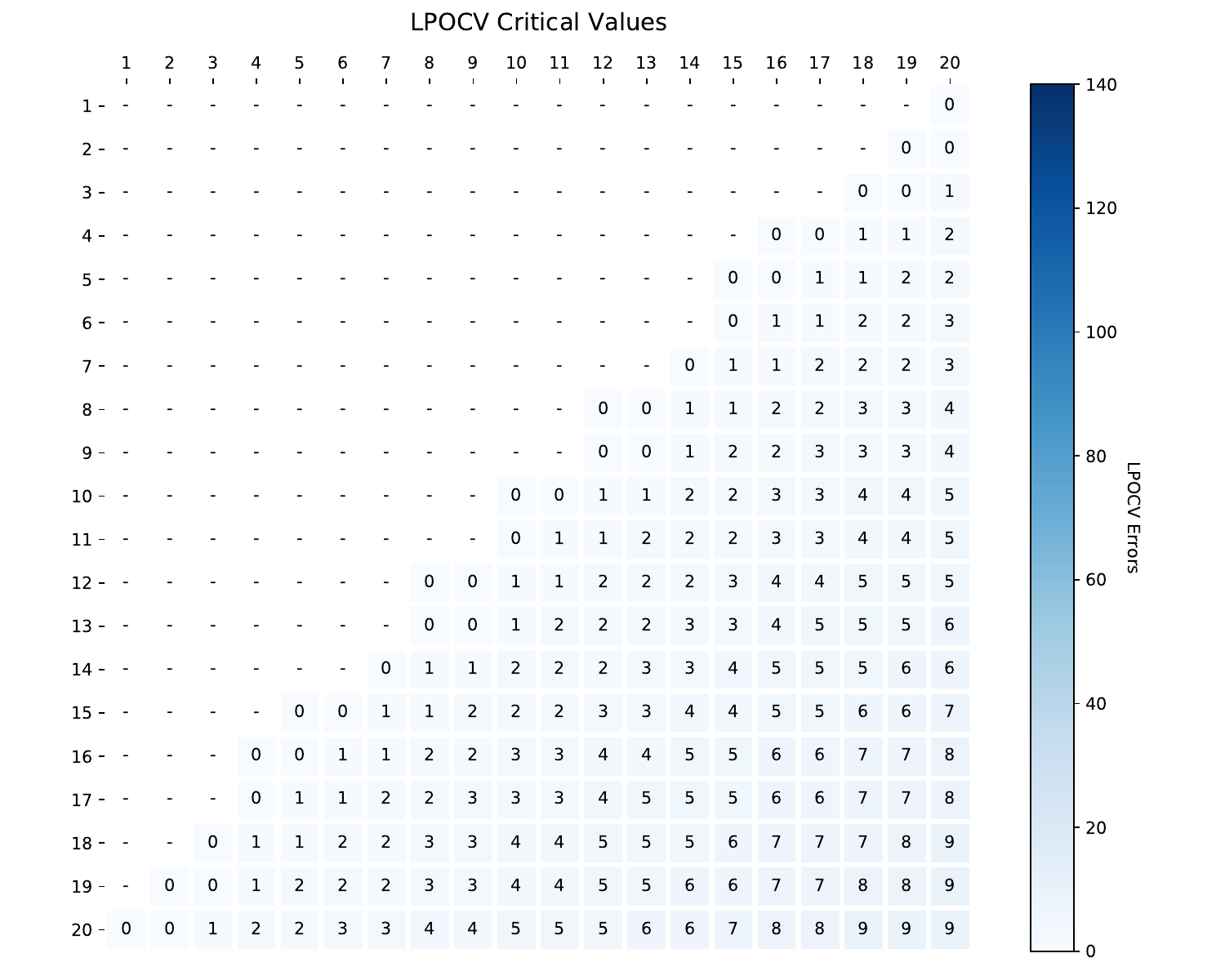}
\caption{Critical values based on the Graham-Sloane type of bound. The color scale is the same as that used to depict the critical values of the ordinary WMW test in Table~\ref{wmwcriticals}.}
\label{grahansloanecriticals}
\end{figure}

\begin{figure}
\centering
\includegraphics[width=0.9\linewidth]{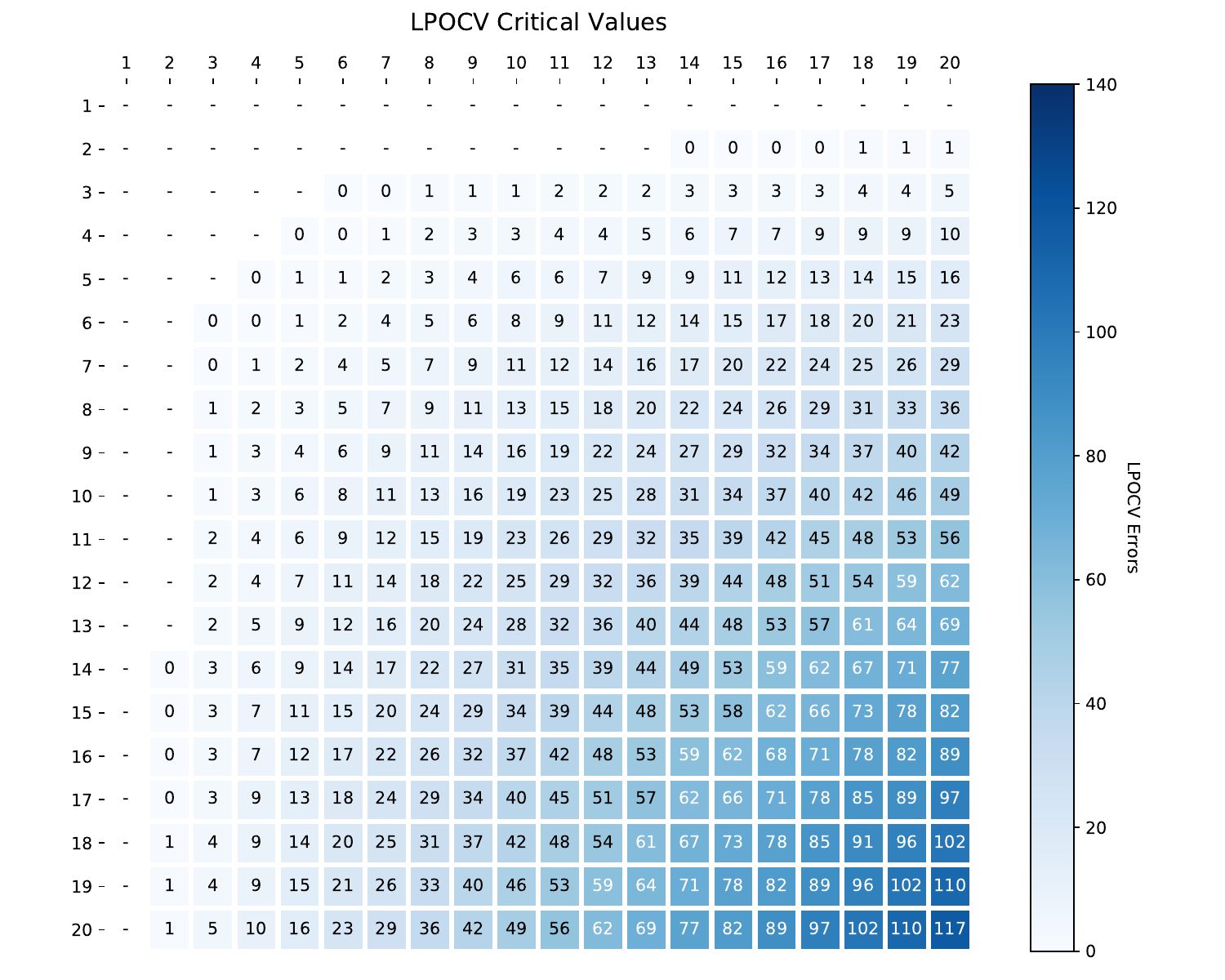}
\caption{Empirical table of critical values obtained using the supremum approach over a limited set of experiments.}
\label{empiricalcriticals}
\end{figure}

Next, we consider the typical behavior of some ``useful'' learning algorithms with both simulated and real-world data. Note that the experiments are not intended to provide an exhaustive empirical evaluation, 
but rather a case study and starting point for more comprehensive investigations that are outside the scope of this paper.

As discussed above, the supremum approach of Theorem~\ref{suptheorem} would provide us the maximal size codes that may never materialize in practical applications, and hence we would end up with an overly conservative test. Instead of taking the supremum over all possible learning algorithms for which the null hypothesis holds, we now explore the usefulness of the following empirically estimated p-value:

\begin{align}\label{empest}
p_\lalg(\labelseq)=\sup_{\lalg'\in C}\operatorname{\widehat{P}}_{\binvar\sim\textnormal{Uni}(\cwwset(\tsetsize,\posc))}(\uval_{\lalg'}(\binvar)\leq\uval_\lalg(\labelseq))\;,
\end{align}
where $C$ is a set of learning algorithms and $\widehat{P}$ is an empirical probability estimate obtained by sampling both $\indset$ and $\binseq$.

To realize the set of learning algorithms, we choose certain representative examples, ridge regression (RR) and k-nearest neighbors (kNN), on different types of simulated data. RR is a typical example of a linear and parametric method, while kNN is both non-linear and non-parametric. For both, cross-validation estimates can be computed very efficiently, which makes running the large number of repetitions needed in the simulations computationally feasible. We recall that a learning algorithm in this paper refers to a combination of the sequence of data $\indset$ and a method such as RR rather than to the RR method only. That is, RR applied on two different sequences of data should be considered as two different learning algorithms. The RR experiments are run using the fast and exact LPOCV algorithm \citep{pahikkala2008exact} implemented in the RLScore library \citep{pahikkala2016rlscore}, for kNN scikit-learn implementation \citep{pedregosa2011scikit} is used.

Following a similar setup for generating data as in our previous study \citep{Montoya2018tlpocv}, we performed a set of simulations on four different distributions of synthetic data. Here, the experiments are carried out with both classes drawn from the same distribution, indicating that the null hypothesis of independence between the class labeling and learning algorithm holds. The first two of the distributions are one and ten dimensional standard normal distributions, and the other two are one and ten dimensional mixtures of two normal distributions. In total, eight different settings for the simulations were considered: ridge regression with regularization parameter value set to 1 and kNN with the number of neighbors set to 3 were both run with the four types of synthetic data distributions. Further details of the data and experiments are available at the repository containing the program codes for the experiments.\footnote{\url{https://gitlab.utu.fi/parmov/u-test-with-leave-pair-out-cross-validation}}

To obtain the critical values for the significance level $0.05$ based on the empirical p-values (\ref{empest}), we let the number of data labeled with 0 and 1 both range from 1 to 20, and experiments with each of the eight learning setups were repeated 10,000 times. In each repetition a new sample with same characteristic was drawn and the LPO errors were counted. The resulting table of critical values is illustrated in Figure~\ref{empiricalcriticals}.

Each setting resulted in a table 20x20 of critical values. These tables were then merged by selecting the smallest value within the experiments with same number of positives and sample size to obtain the LPO critical values in Figure~\ref{empiricalcriticals}.  

\subsection{Quality of LPO critical values}

We evaluated the quality of LPO critical values in rejecting a false null hypothesis by performing a set of experiments with linearly and non-linearly separable synthetic data, as well as using real medical data. These experiments consisted in computing the proportion of type II error (i.e. failure to reject a false null hypothesis) for LPO critical values in settings where ridge regression and kNN were used on both synthetic and real data.

In the experiments, we used eight different sample sizes (i.e. 12, 16, 20, 24, 28, 32, 36 and 40) with the fraction of positives and negatives set to 50\%. In the linearly separable synthetic data, we considered 10 features with one or four features having signal. In the signal feature, one class was drawn from a normal distribution with 0.5 mean and variance 1, while for the other class the mean was -0.5. For the non-linearly separable data sets, the feature with signal was drawn from three normal distributions. The first class was drawn from a normal distribution with mean 0.5 and variance 1 and the second class either from a normal distribution with mean 5.5 or with -4.5 mean equally likely. 

In addition to the synthetic data, we also performed experiments with a real medical data set involving prostate magnetic resonance imaging (MRI). More precisely, the data consisted of diffusion weighted imaging (DWI), a widely used modality in detecting prostate cancer, of 20 patients with histologically confirmed prostate cancer in the peripheral zone. This data has been used in prior studies \citep{Jambor2015Evaluation,toivonen2015mathematical,merisaari2015mathematical} for development and validation of post-processing methods. In these experiments, we used this data to classify DWI voxels as cancerous or non-cancerous. The DWI data set consisted of 85876 voxels, of which 9268 were labeled positive, obtained from parametric maps of the 20 patients with confirmed prostate cancer. Each voxel was associated with six features that are known to be linked with cancer \citep{toivonen2015mathematical,merisaari2015mathematical,montoya2015detecting,Langer2009,Ginsburg2011Variable}. For more details on the data, we refer to our previous study \citet{Montoya2018tlpocv}.

Figure~\ref{typeIIerrorplots} presents the results of our experiments in the four different settings previously described. In case of synthetic linearly separable data (Figure \ref{typeIIerrorplots}.a and Figure \ref{typeIIerrorplots}.b), the proportion of type II error decreases as the sample size increases as expected. However, the amount of reduction of the proportion of type II error also depends on the algorithm used and the strength of the signal in the data.  In contrast, when the signal of the data is non-linearly separable (Figure \ref{typeIIerrorplots}.c) the proportion of type II error is high when a linear algorithm (i.e. ridge regression) is used independently of the size or signal in the data, opposed to using a non-linear algorithm (i.e. kNN). In a real-life setting as it is the medical data set (Figure \ref{typeIIerrorplots}.d), the proportion of type II error behaves similarly than the synthetic data with strong linearly separable data (Figure \ref{typeIIerrorplots}.b).

\begin{figure}
\centering
\includegraphics[width=0.95\linewidth]{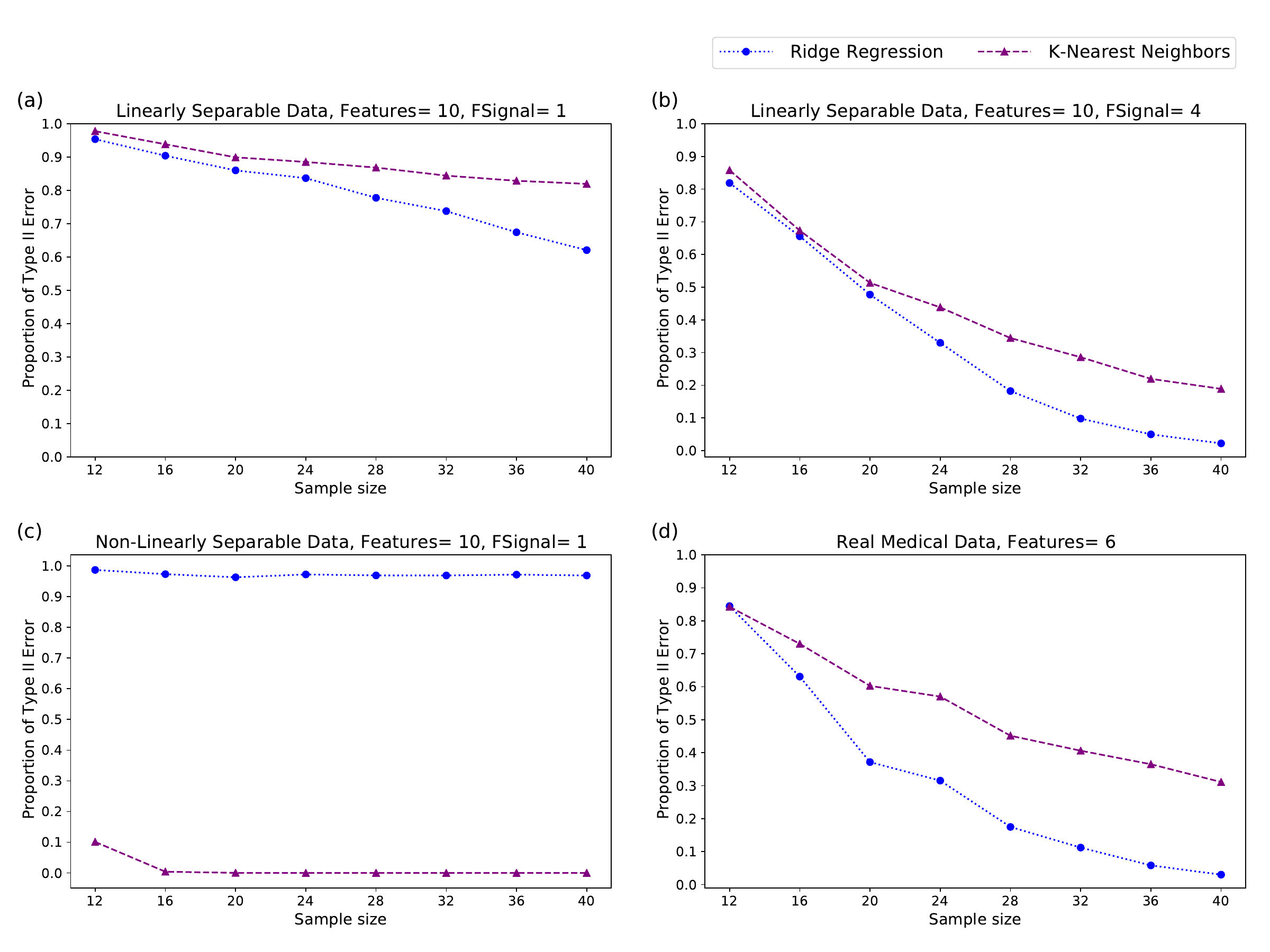}
\caption{LPOCV Proportion of Type II Errors.}
   \label{typeIIerrorplots}
\end{figure}

\section{Discussion and Conclusions}\label{disc}

We have explored the intuitive analogy between cross-validation for learning algorithms and error detecting codes and show that the analogy is also mathematically exact. In particular, we have shown that the maximal sizes of constant weight codes with Hamming distance four between code words coincide exactly with the maximal number of different binary label assignments with fixed label proportion for which a learning algorithm is able to get zero leave-pair-out cross-validation (LPOCV) error. With the help of coding and graph theoretic tools, we have extended this analysis by introducing what we call $\uu$-light constant weight codes that allow at most $\uu$ errors, and shown how these coincide with maximal number of label assignments  leading to at most $\uu$ LPOCV errors. The light codes are based on the more general analogy that, for a sample of data and its arbitrary binary labeling with a fixed label proportion, all possible LPOCV outcomes of a learning algorithm can be considered as an orientation of a Johnson graph with vertices corresponding to each possible fixed proportion binary labeling of the sample. The number of LPOCV errors of a learning algorithm for a labeling of the sample coincides with the outdegree of the corresponding Johnson graph vertex. This makes it possible to cast the analysis of learning algorithms' LPOCV behavior to a graph orientation problem, and to use graph theoretic tools to bound the maximal sizes of considered light constant weight codes.

The above results open up several directions for further research. As shown in this paper, one can design statistical tests resembling the classical Wilcoxon-Mann-Whitney test for the null hypothesis of the fixed proportion sample labeling being randomly assigned. The properties of this type of a test were analysed and also experimented with a case study involving simulations and a real data set. The behavior of typical learning algorithms in this context is a subject for more extensive studies in the future. In particular, since our analysis made no prior assumptions of the properties of the learning algorithms or the data, making such assumptions may provide us more powerful significance tests than those based on the worst case analysis known in the statistical literature as the supremum approach. In particular, since the learning algorithms can be analyzed based on their LPOCV capacities, one could ask whether some algorithm classes are genuinely better than others in this sense. The most of the Johnson graph orientations have the vertices' outdegree distribution heavily concentrated around the mean, indicating that the corresponding  learning algorithms can not have good LPOCV results with any labeling of the data. In the other extreme, one can intentionally design a learning algorithm corresponding to a light constant weight code, that has disproportionally high changes for getting a specific good LPOCV result under a random label assignment. On the other hand, as was shown with simulations, the outdegree distributions of the practical algorithms tend to be spread wider than the ordinary Wilcoxon distribution, but are still far from the extraordinary distributions corresponding to the light constant weight codes. Similarly to the classical learning theoretic analysis of learning algorithms with Vapnik-Chervonenkis dimension \citep{vapnik1995statistical} or their stability properties \citep{shalev2010learnability}, one could also analyze the LPOCV capacity of some specific algorithm classes or specific data, which could again lead to more powerful tests if the prior assumptions would lead to considerably tighter restrictions of the corresponding code types.

In this paper, we only covered LPOCV and the analogous types of constant weight codes with Hamming distance four between code words. However, similar analogues could as well be drawn between cross-validation types with larger hold-out sets than in LPOCV and codes with Hamming distance larger than four. We expect that, since the maximal sizes of such codes would be considerably smaller than those corresponding to LPOCV, this would also provide us possibilities for designing more powerful tests. This is also intuitive as the codes able to correct errors consisting of several bits are much smaller in size than the codes only correcting small errors, and the sizes of the hold-out sets correspond to the number of bits the code should correctly set. However, this research direction would require quite different mathematical tools than those considered in this paper, and hence they are left for future studies.

\section*{Acknowledgements}
This work has been funded by Research Council of Finland (grants 340182, 345804, 340140, 345805).

\bibliography{myBibliography}

\end{document}